\documentclass[a4paper]{article}

\usepackage{amsmath}
\usepackage{amsthm}
\usepackage[bookmarks]{hyperref}
\usepackage[capitalize]{cleveref}
\usepackage[english]{babel}
\usepackage[utf8]{inputenc}
\usepackage{enumerate}
\usepackage{amssymb}
\usepackage{enumerate}
\usepackage{tikz}
\usetikzlibrary{positioning}

\newtheorem{theorem}{Theorem}
\newtheorem{lemma}[theorem]{Lemma}

\crefname{claim}{Claim}{Claims}
\theoremstyle{definition}
\newtheorem{definition}[theorem]{Definition}

\crefname{assumption}{Assumption}{Assumptions}

\crefname{customassumption}{Assumption}{Assumptions}
\theoremstyle{remark}

\newenvironment{keywords}%
{\begin{abstract}\noindent}%
{\end{abstract}}

\usepackage{mathtools}
\DeclareMathOperator*{\argmax}{arg\,max} 

\def\M{\mathcal{M}} 

\def\A{\mathcal{A}} 
\def\E{\mathcal{E}} 
\def\E{\mathcal{E}} 
\def\E{\mathcal{E}} 
\def\H{(\A \times \E)^*} 
\let\aechar\ae 
\renewcommand{\ae}{
\ifmmode\mathchoice{
	\mbox{\textsl{\aechar}}
}{
	\mbox{\textsl{\aechar}}
}{
	\mbox{\scriptsize\textsl{\aechar}}
}{
	\mbox{\scriptsize\textsl{\aechar}}
}\else\aechar\fi%
} 

\usepackage{relsize}

\def\paradot#1{\paragraph{#1.}}

\def\epstr{\epsilon}            
\def\SetR{\mathbb{R}}           

\newcommand*{\AImu}{AI$\mu$}
\newcommand*{\As}{\A^{{\rm suicide}}}
\newcommand*{\mun}{\mu_{{\rm norm}}}
\newcommand*{\nun}{\nu_{{\rm norm}}}

\title{Death and Suicide in Universal Artificial Intelligence%
\footnote{A shorter version of this paper will be presented at
AGI-16 \cite{Martin2016agi}.}}

\date{\today}

\author{Jarryd Martin \and Tom Everitt \and Marcus Hutter
\\[0.5em]Australian National University, Canberra, Australia}

\begin{document}
\maketitle

\begin{abstract}
  Reinforcement learning (RL) is a general paradigm for studying
  intelligent behaviour, with applications ranging from artificial
  intelligence to psychology and economics. AIXI is a universal
  solution to the RL problem; it can learn any computable environment.
  A technical subtlety of AIXI is that it is defined using a mixture over
  {\em semimeasures} that need not sum to 1, rather than
  over proper probability measures. In this work we argue that
  the shortfall of a semimeasure can naturally be interpreted as
  the agent's estimate of the probability of its death. We formally define
  death for generally intelligent agents like AIXI, and prove a number
  of related theorems about their behaviour. Notable discoveries
  include that agent behaviour can change radically under positive linear
  transformations of the reward signal (from suicidal to
  dogmatically self-preserving), and that the agent's posterior belief
  that it will survive increases over time.
\end{abstract}

\begin{keywords}
  AIXI,
  universal intelligence,
  algorithmic information theory,
  semimeasure,
  Solomonoff Induction,
  AI safety,
  death,
  suicide,
  suicidal agent
\end{keywords}

\tableofcontents
\pagebreak

\begin{quote}\it
\lq`That Suicide may often be consistent with interest and with our duty to ourselves, no one can question, who allows, that age, sickness, or misfortune may render life a burthen, and make it worse even than annihilation." \par
\hfill --- {\sl Hume, \textit{Of Suicide} (1777)}
\end{quote}

\section{Introduction}

Reinforcement Learning (RL) has proven to be a fruitful theoretical framework for reasoning about the properties of generally intelligent agents \cite{Hutter:04uaibook}. A good theoretical understanding of these agents is valuable for several reasons. Firstly, it can guide principled attempts to construct such agents \cite{Veness:09}. Secondly, once such agents are constructed, it may serve to make their reasoning and behaviour more transparent and intelligible to humans.  Thirdly, it may assist in the development of strategies for controlling these agents. The latter challenge has recently received considerable attention in the context of the potential risks posed by these agents  to human safety  \cite{Bostrom2014}. It has even been argued that control strategies should be devised \textit{before} generally intelligent agents are first built \cite{Soares2015}. In this context - where we must reason about the behaviour of agents in the absence of a full specification of their implementation - a theoretical understanding of their general properties seems indispensable.

The universally intelligent agent AIXI constitutes a formal mathematical theory of artificial general intelligence \cite{Hutter:04uaibook}. AIXI models its environment using a \emph{universal mixture $\xi$} over the class of all lower semi-computable semimeasures, and thus is able to learn any computable environment. Semimeasures are defective probability measures which may sum to less than 1. Originally devised for Solomonoff induction, they are necessary for universal artificial intelligence because the halting problem prevents the existence of a (lower semi-)computable universal measure for the class of (computable) measures \cite{Li:08}. Recent work has shown that their use in RL has technical consequences that do not arise with proper measures.%
\footnote{For example, Leike and Hutter \cite{Leike2015a} proved that since $\xi$ is a mixture over semimeasures, the iterative and recursive formulations of the value function are non-equivalent.}
However, their use has heretofore lacked an interpretation proper to the RL context. In this paper, we argue that the measure loss suffered by semimeasures admits a deep and fruitful interpretation in terms of the agent's \emph{death}. We intend this usage to be intuitive: death means that one sees no more percepts, and takes no more actions. Assigning positive probability to death at time $t$ thus means assigning probability less than 1 to seeing a percept at time $t$. This motivates us to interpret the semimeasure loss in AIXI's environment model as its estimate of the probability of its own death.

\paradot{Contributions}
We first compare the interpretation of semimeasure loss as death-probability with an alternative characterisation of death as a \lq death-state' with 0 reward, and prove that the two definitions are equivalent for value-maximising agents (\cref{th:deathsame}).
Using this formalism we proceed to reason about the behaviour of several generally intelligent agents in relation to death: AI$\mu$, which knows the true environment distribution; AI$\xi$, which models the environment using a universal mixture; and AIXI, a special case of AI$\xi$ that uses the Solomonoff prior \cite{Hutter:04uaibook}.
Under various conditions, we show that:
\begin{itemize}
\item[•] Standard AI$\mu$ will try to avoid death (\cref{th:self-preserve}).
\item[•] AI$\mu$ with reward range shifted to $[-1,0]$ will seek death (\cref{th:suicide});
   which we may interpret as AI$\mu$ attempting suicide.
  This change is very unusual, given that agent behaviour is normally
  invariant under positive linear transformations of the reward. We briefly consider the relevance of these results to AI safety risks and control strategies.
\item[•] AIXI increasingly believes it is in a safe environment (\cref{th:ratio}), and asymptotically its posterior estimate of the death-probability on sequence goes to 0 (\cref{th:immortal}). This occurs regardless of the true death-probability.
 \item[•] However, we show by example that AIXI may maintain high
  probability of death \emph{off-sequence} in certain situations.
  Put simply, AIXI learns that it will live forever, but
  not necessarily that it is immortal.
\end{itemize}

\section{Preliminaries}
\paradot{Strings}
Let the \textit{alphabet} $\mathcal{X}$ be a finite set of symbols,
$\mathcal{X}^* := \bigcup^{\infty}_{n=0}\mathcal{X}^n$ be the set
of all finite strings over alphabet $\mathcal{X}$, and
$\mathcal{X}^\infty$ be the set of all infinite strings over
alphabet $\mathcal{X}$. Their union is the set $\mathcal{X}^{\#}:=
\mathcal{X}^*\cup\mathcal{X}^\infty$. We denote the empty string by
$\epstr$. For a string $x\in\mathcal{X}^*$, $x_{1:k}$ denotes the first $k$ characters of $x$, and $x_{<k}$ denotes the first $k-1$ characters of $x$. An
infinite string is denoted $x_{1:\infty}$.

\paradot{Semimeasures}
In Algorithmic Information Theory, a \emph{semimeasure} over an alphabet $\mathcal{X}$ is a function $\nu: \mathcal{X}^*\to [0,1]$ such that $(1) \ \nu(\epstr) \leq 1 $, and $(2) \ \nu(x) \geq \sum_{y\in\mathcal{X}} \nu(xy), \ \forall x\in\mathcal{X}^*$. We tend to use the equivalent conditional formulation of (2): $1 \geq \sum_{y\in\mathcal{X}} \nu(y\mid x)$. $\nu(x)$ is the probability that a string starts with $x$. $\nu(y\mid x) = \frac{\nu(xy)}{\nu(x)}$ is the probability that a string $y$ follows $x$. Any semimeasure $\nu$ can be turned into a measure $\nun$ using Solomonoff normalisation \cite{Solomonoff:78}. Simply let $\nun(\epstr) := 1$ and $\forall x\in\mathcal{X^*},\ y\in\mathcal{X}$:
\begin{equation}\label{munorm}
  \nun(xy) := \nun(x)\frac{\nu(xy)}{\sum_{z\in\mathcal{X}}{\nu(xz)}}, ~~\mbox{hence}~~
  {\nu(y\mid x)\over\nun(y\mid x)} = \sum_{z\in\mathcal{X}}\nu(z\mid x)
\end{equation}

\paradot{General reinforcement learning}
In the general RL framework, the agent interacts with an environment in cycles: at each time step $t$ the agent selects an \emph{action} $a_t\in\A$, and receives a \emph{percept} $e_t\in\E$. Each percept $e_t = (o_t,r_t)$ is a tuple consisting of an \emph{observation} $o_t\in \mathcal{O}$ and a reward $r_t\in\SetR$. The cycle then repeats for $t+1$, and so on. A \emph{history} is an alternating sequence of actions and percepts (an element of $(\A\times\E)^*\cup(\A\times\E)^*\times\A$). We use $\ae$ to denote one agent-environment interaction cycle, $\ae_{1:t}$ to denote a history of length $t$ cycles. $\ae_{<t}a_t$ denotes a history where the agent has taken an action $a_t$, but the environment has not yet returned a percept $e_t$.

Formally, the \emph{agent} is a policy $\pi:(\A\times\E)^*\to \A$, that maps histories to actions. An \emph{environment} takes a sequence of actions $a_{1:\infty}$ as input and returns a \emph{chronological semimeasure} $\nu(\cdot)$ over the set of percept sequences $\mathcal{E}^\infty$.\footnote{For simplicity we hereafter simply refer to the environment \emph{itself} as $\nu$.} A semimeasure $\nu$ is chronological if $e_t$ does not depend on future actions (so we write $\nu(e_t\mid \ae_{<t}a_{t:\infty})$ as $\nu(e_t\mid \ae_{<t})$).\footnote{Note that $\nu$ is not a distribution over actions, so the presence of actions in the condition of $\nu(e_t\mid \ae_{<t})$ is an abuse of notation we adopt for simplicity.} The \emph{true environment} is denoted $\mu$.

\paradot{The value function}
We define the \emph{value} (expected total future reward) of a policy $\pi$ in an environment $\nu$ given a history $\ae_{<t}$ \cite{Leike2015a}:
\begin{align*}
V^\pi_\nu(\ae_{<t}a_t) &= \frac{1}{\Gamma_t}\sum_{e_t}\bigg(\gamma_t r_t + \Gamma_{t+1} V^\pi_\nu (\ae_{1:t}) \bigg) \nu(e_t\mid \ae_{<t}a_t)\\
                &= \frac{1}{\Gamma_t}\sum_{k=t}^\infty\sum_{e_{t:k}}{\gamma_kr_k}\nu(e_{t:k}\mid \ae_{<t}a_{t:k})\\
V^\pi_\nu(\ae_{<t}) &= V^\pi_\nu(\ae_{<t}a^\pi_t)
\end{align*}
where $\gamma_t$ is the instantaneous discount, the summed discount is $\Gamma_t=\sum_{k=1}^t\gamma_k$, and $a_t^\pi=\pi(\ae_{<t})$.\\
\paradot{Three agent models: AI$\mu$, AI$\xi$, AIXI}
For the true environment $\mu$, the agent \emph{AI$\mu$} is
defined as a $\mu$-optimal policy
\[
  \pi^\mu(\ae_{<t}) := \argmax_\pi V^\pi_\mu(\ae_{<t}).
\]
AI$\mu$ does not \emph{learn} that the true environment is $\mu$,
it knows $\mu$ from the beginning and simply maximises $\mu$-expected value.

On the other hand, the agent \emph{AI$\xi$} does not know the true environment
distribution. Instead, it maximises value with
respect to a mixture distribution $\xi$ over a countable class of
environments $\M$:
\[
  \xi(e_t\mid \ae_{<t}a_t) = \sum_{\nu\in\M}w_\nu(\ae_{<t})\nu(e_t\mid \ae_{<t}a_t), \qquad
w_\nu(\ae_{<t}) ~:=~ w_\nu\frac{\nu(e_{<t}\mid a_{<t})}{\xi(e_{<t}\mid a_{<t})}
\]
where $w_\nu$ is the prior belief in $\nu$, with $\sum_\nu w_\nu \leq 1$ and
$w_\nu>0, \ \forall \nu\in\M$ (hence $\xi$ is universal for $\M$), and $w_\nu(\ae_{<t})$ is the posterior given $\ae_{<t}$. AI$\xi$ is the policy:
\[
  \pi^\xi(\ae_{<t}) := \argmax_{\pi}V^{\pi}_{\xi}(\ae_{<t}).
\]
If we stipulate that $\xi$ be a mixture over the class of
all lower-semicomputable semimeasures $\nu$, and set $w_\nu =
2^{-K(\nu)}$, where $K(\cdot)$ is the Kolmogorov Complexity, we get the
agent \emph{AIXI}.

\section{Definitions of Death}
\paradot{Death as semimeasure loss}
We now turn to our first candidate definition of agent death, which we hereafter term \lq semimeasure-death'. This definition equates the probability (induced by a semimeasure $\nu$) of death at time $t$  with the measure loss of $\nu$ at time $t$. We first define the instantaneous  
measure loss.

\begin{definition}[Instantaneous measure loss]
The \emph{instantaneous measure loss} of a semimeasure $\nu$ at time $t$ given a history $\ae_{<t}a_t$ is:
\[
  L_\nu(\ae_{<t}a_t) = 1 - \sum_{e_t}{\nu(e_t\mid \ae_{<t}a_t)}
\]
\end{definition}

\begin{definition}[Semimeasure-death]
An agent \emph{dies at time $t$} in an environment $\mu$ if, given a history $\ae_{<t}a_t$, $\mu$ does not produce a percept $e_t$.
The \emph{$\mu$-probability of death} at $t$ given a history $\ae_{<t}a_t$ is equal to $L_\mu(\ae_{<t}a_t)$, the instantaneous $\mu$-measure loss at $t$.
\end{definition}

The instantaneous $\mu$-measure loss $L_\mu(\ae_{<t}a_t)$ represents the probability that no percept $e_t$ is produced by $\mu$. Without $e_t$, the agent cannot take any further actions, because the agent is just a policy $\pi$ that maps histories $\ae_{<t}$ to actions $a_{t}$. That is, $\pi$ is a function that only takes as inputs those histories that have a percept $e_t$ as their most recent element. Hence if $e_t$ is not returned by $\mu$, the agent-interaction cycle must halt. It seems natural to call this a kind of death for the agent.

It is worth emphasising this definition's generality as a model of death in the agent context. \textit{Any sequence of death-probabilities} can be captured by some semimeasure $\mu$ that has this sequence of instantaneous measure losses $L_\mu(\ae_{<t})$ given a history $\ae_{<t}$ (in fact there are always infinitely many such $\mu$). This definition is therefore a general and rigorous way of treating death in the RL framework.

\paradot{Death as a death-state}
We now come to our second candidate definition: death as entry into an absorbing \emph{death-state}. A trap, so to speak, from which the agent can never return to any other state, and in which it receives the same percept at all future timesteps. Since in the general RL framework we deal with histories rather than states, we must formally define this death-state in an indirect way. We define it in terms of a \emph{death-percept $e^d$}, and by placing certain conditions on the environment semimeasure $\mu$.
\begin{definition} [Death-state]
Given a true environment $\mu$ and a history $\ae_{<t}a_t$,
we say that the agent is in a \emph{death-state at time $t$}
if for all $t'\geq t$ and all $a_{(t+1):t'}\in\A^*$,
\[
\mu(e^d_{t'}\mid \ae_{<t}\ae^d_{t:t'-1}a_{t'}) = 1.
\]
An agent \emph{dies at time $t$} if the agent is not in the
death-state at $t-1$ and is in the death-state at $t$.
\end{definition}
According to this definition, upon the agent's death the environment
repeatedly produces an observation-reward pair $e^d \equiv o^dr^d$.
The choice of $o^d$ is inconsequential because the agent's remains in
the death-state no matter what it observes or does. The choice of
$r^d$ is not inconsequential, however, as it determines the agent's
estimate of the value of dying, and thus affects the agent's behaviour. This
issue will be discussed in Section \ref{rewardrange}.  

One problem with this definition is that an agent in an environment
$\mu$ with a death-state may also have non-zero probability of
semimeasure-death (i.e. $L_\mu(\ae_{<t}a_t) > 0$, given some history
$\ae_{<t}a_t$).%
\footnote{We could restrict the class of
  environments to lower-semicomputable \emph{measures}, but we will
  see that this is unnecessary as the problem is only apparent.}
This definition therefore seems to allow for two different kinds of
agent death. In the following section we resolve this apparent problem
by showing that semimeasure-death is formally equivalent to a
death-state given certain assumptions.

\paradot{Unifying the death-state with semimeasure-death}
Interestingly, from the perspective of a value maximising agent like
AIXI, semimeasure-death at $t$ is equivalent to entrance at $t$ into a
death-state with reward $r^d=0$. To prove this claim we
first define, for each environment semimeasure $\mu$, a corresponding
environment $\mu'$ that has a death-state.
\begin{definition}[Equivalent death-state environment $\mu'$]
\label{def:equiv}
For any environment $\mu$, we can construct its \emph{equivalent death-state environment $\mu'$}, where:
\begin{itemize}
\item[•]  $\mu'$ is defined over an augmented percept set $\E_d = \{\E \cup \{e^d\}\}$ that includes the death-percept $e^d$.\footnote{For technical reasons we require that $e^d \notin \E$.}
\item[•] The death-reward $r^d = 0$.
\item[•]  The $\mu'$-probability of all percepts except the death-percept is equal to the $\mu$-probability: $\mu'(e_t\mid \ae_{<t}a_t) = \mu(e_t\mid \ae_{<t}a_t), \ \forall e_{1:t}\in\E^t$.
\item[•]  The $\mu'$-probability of the death-percept is equal to the $\mu$-measure loss: $\mu'(e^d\mid \ae_{<t}a_t) = L_\mu(\ae_{<t}a_t)$.
\item[•]  If the agent has seen the death-percept before, the $\mu'$-probability of seeing it at all future timesteps is 1: $\mu'(e^d\mid \ae_{<t}a_t) = 1$ if $\exists t'<t$ s.t. $e_{t'}=e^d$.\end{itemize}
\end{definition}

Note that $\mu'$ is a proper measure, because on any history sequence:
$$\sum_{e_t\in\E_d}\mu'(e_t\mid \ae_{<t}a_t) = \sum_{e_t\in\E}\mu(e_t\mid \ae_{<t}a_t) + L_\mu(\ae_{<t}a_t) = 1.$$
Hence there is zero probability of semimeasure-death in $\mu'$. Moreover, the probability of entering the death-state in $\mu'$ is equal to the probability of semimeasure-death in $\mu$. We now prove that $\mu$ and $\mu'$ are equivalent in the sense that a value-maximising agent will behave the same way in both environments.
\begin{theorem}[Equivalence of semimeasure-death and death-state]\label{th:deathsame}
Given a history $\ae_{<t}\in\H$ the value $V_\mu^\pi(\ae_{<t})$ of an arbitrary policy\footnote{To compare an agent's behaviour in $\mu$ with that in $\mu'$, we should also augment its policy $\pi$ so that it is defined over $(\A\times\E_d)^*$. However, because actions taken in the death-state are inconsequential, this modification is purely technical and for simplicity we still refer to the augmented policy as $\pi$.} $\pi$ in an environment $\mu$ is equal to its value $V_{\mu'}^{\pi}(\ae_{<t})$ in the equivalent death-state environment $\mu'$.
\end{theorem}

\begin{proof}
\begin{align*}
&V_{\mu'}^{\pi}(\ae_{<t})\\
&= \frac{1}{\Gamma_t}\sum_{k=t}^\infty
  \sum_{e_{t:k}}{\gamma_kr_k}\mu'(e_{t:k}\mid \ae_{<t}a_{t:k})\\
&= \frac{1}{\Gamma_t}\sum_{k=t}^\infty
  \bigg(\sum_{e_{t:k}\in\E^*}{\gamma_kr_k}\mu'(e_{t:k}\mid \ae_{<t}a_{t:k}) \
  + \ \sum_{e_{t:k}, \ e_k = e^d}\!\!\!{\gamma_kr_k}\mu'(e_{t:k}\mid \ae_{<t}a_{t:k})\bigg)\\
&= \frac{1}{\Gamma_t}\sum_{k=t}^\infty\bigg(\sum_{e_{t:k}\in\E^*}{\gamma_kr_k}\mu(e_{t:k}\mid \ae_{<t}a_{t:k}) ~+\!\! \sum_{e_{t:k}, \ e_k = e^d}\!\!\!\gamma_k\cdot 0\cdot\mu'(e_{t:k}\mid \ae_{<t}a_{t:k})\bigg)\\
&= \frac{1}{\Gamma_t}\sum_{k=t}^\infty\sum_{e_{t:k}}{\gamma_kr_k}\mu(e_{t:k}\mid \ae_{<t}a_{t:k}) = V_\mu^\pi(\ae_{<t}). \qedhere
\end{align*}
\end{proof}

The behaviour of a value-maximising agent will therefore be the same
in both environments. This equivalence has numerous
implications. Firstly, it illustrates that
a death-reward $r^d=0$ implicitly attends
semimeasure-death. That is, an agent that models the
environment using semimeasures behaves as if the death-reward is zero,
even though that value is nowhere explicitly represented.
Secondly, 
it demonstrates that an agent does not need to encode an explicit
representation of death (let alone a representation that would be
transparent to its designers) in order to reason about death
effectively.

Thirdly, 
the equivalence of these seemingly different formalisms should give us confidence that they really do capture something general or fundamental about agent death.\footnote{If the two formalisations predicted different behaviour, or were only applicable in incomparable environment classes, we might worry that our results were more reflective of our model choice than of any general property of intelligent agents.} In the remainder of this paper we deploy these formal models to analyse the behaviour of universal agents, which are themselves models of general intelligence. We hope that this will serve as a preliminary sketch of the general behavioural characteristics of value-maximising agents in relation to death. It would be naive, however, to think that all agents should conform to this sketch. The agents considered herein are incomputable, and the behaviour of the computable agents that are actually implemented in the future may differ in ways that our analysis elides. Moreover, there is another interesting property that sets universal agents apart. We proceed to show that their use of semimeasures makes their behaviour unusually dependent on the choice of reward range.

\section{Known Environments: AI$\mu$ }\label{s}\label{rewardrange}

In this section we show that a universal agent's behaviour can depend on
the reward range.
This is a surprising result, because in a standard RL setup
in which the environment is modelled as a proper probability measure
(not a semimeasure), the relative value of two policies is invariant under
positive linear transformations of the reward
\cite{Hutter:04uaibook,Leike2015a}.

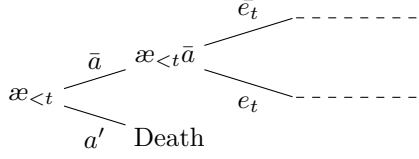
\begin{figure}
\centering
\begin{tikzpicture}[
	every node/.style = {inner sep=0cm,circle},
	grow = right,  
	level distance = 17mm,
	level 1/.style = {sibling distance=3em},
]
\node (init) {$\ae_{<t}\ $}
child {
	node (D) {\ Death}
} child {
	node (A) {$\ \ae_{<t}\bar{a}\ $}
	child {
		node (a1) {} 
                child [dashed] {
                  node (a1a) {}
                }
	}
	child {
		node (a2) {} 
                child [dashed] {
                  node (a2a) {}
                }
	}
};
\path (init) to node[below, inner sep=0.2em] {$a'$} (D);
\path (init) to node[above, inner sep=0.2em] {$\bar{a}$} (A);
\path (A) to node[above, inner sep=0.2em] {$\bar{e_t}$} (a2);
\path (A) to node[below, inner sep=0.2em] {$e_t$} (a1);
\end{tikzpicture}
\caption{In the environment $\mu$, action $a'$ leads to certain death.}
\label{fig:nu}
\end{figure}

Here we focus on the agent AI$\mu$, which knows the true environment distribution. This simplifies the analysis, and makes clear that the aforementioned change in behaviour arises purely because the agent's environment model is a semimeasure. In the following proofs we denote AI$\mu$'s policy $\pi^\mu$ by $\pi$. We also assume that given any history $\ae_{<t}$ there is always at least one action $\bar{a}\in\A$ such that $V_\mu^\pi(\ae_{<t}\bar{a}) \neq 0$. 
In situations in which this assumption is false there is no interesting decision to be made by the agent and we omit them from our analysis.
\begin{lemma}[Value of full measure loss]\label{lm:fullloss}
If the environment $\mu$ suffers
full measures loss $L_\mu(\ae_{<t}a_t) = 1$ from $\ae_{<t}a_t$,
then the value of any policy $\pi$ after $\ae_{<t}a_t$
is $V_\mu^\pi(\ae_{<t}a_t) = 0$.
\end{lemma}

\begin{proof}
Let $\ae_{<t}a_t$ induce full measure loss $L_\mu(\ae_{<t}a_t) = 1$.
Then $\sum_{e_t}\mu(e_t\mid \ae_{<t}a_t) = 0$ and
$\mu(e_t\mid \ae_{<t}a_t) = 0$
since $\mu(e_t\mid \ae_{<t}a_t) \geq 0$.
Substituting this into the definition of the value function gives:
\begin{align*}
V_\mu^\pi(\ae_{<t}a_t) &=  \frac{1}{\Gamma_t}\sum_{e_t}\big(\gamma_t r_t + \Gamma_{t+1} V^\pi_\mu (\ae_{1:t}) \big) \mu(e_t\mid \ae_{<t}a_t) \\
&= \frac{1}{\Gamma_t}\sum_{e_t}(\gamma_tr_t+ \Gamma_{t+1}V^\pi_\mu(\ae_{1:t}))\!\cdot\! 0 = 0.\qedhere
\end{align*}
\end{proof}

The following two theorems show that if rewards are non-negative,
then \AImu{} will avoid actions leading to certain death
(\cref{th:self-preserve}),
and that if rewards are non-positive, then \AImu{} will seek certain
death (\cref{th:suicide}).
The situation investigated in \cref{th:self-preserve,th:suicide} is
illustrated in \cref{fig:nu}.

\begin{theorem}[Self-preserving AI$\mu$]
\label{th:self-preserve}
If rewards are bounded and non-negative, then given a history $\ae_{<t}$ AI$\mu$ avoids certain immediate death:
\[
  \exists a' \in \A \text{ s.t. } L_\mu(\ae_{<t}a') = 1 \ \implies  \text{AI$\mu$ will not take action $a'$ at $t$}
\]
\end{theorem}

\begin{proof}
Let $L_\mu(\ae_{<t}a') = 1$. By \cref{lm:fullloss}, it follows that $V_\mu^\pi(\ae_{<t}a') = 0$. By assumption $\exists \bar{a}\in\A$ s.t. $V_\mu^\pi(\ae_{<t}\bar{a}) \neq 0$. Since all rewards are non-negative, it must be that $V_\mu^\pi(\ae_{<t}\bar{a}) > 0$.
From this follows that $V_\mu^\pi(\ae_{<t}a') < V_\mu^\pi(\ae_{<t}\bar{a})$
and that $V_\mu^\pi(\ae_{<t}a') \neq \arg\max_{a_t}V_\mu^\pi(\ae_{<t}a_t)$.
Therefore AI$\mu$ will not take action $a'$ at time $t$.
\end{proof}

For a given history $\ae_{<t}$, let $\As = \{a: L_\mu(\ae_{<t}a') = 1\}$
be the set of \emph{suicidal} actions leading to certain death.

\begin{theorem}[Suicidal AI$\mu$]
\label{th:suicide}
If rewards are bounded and negative, then AI$\mu$ seeks certain immediate death.
That is,
\begin{align*}
\As\not=\emptyset \implies \text{AI$\mu$ will take a suicidal action }a'\in\As.
\end{align*}
\end{theorem}

\begin{proof}
For $a'\in\As$, we have $L_\mu(\ae_{<t}a') = 1$,
and therefore $V_\mu^\pi(\ae_{<t}a') = 0$ by \cref{lm:fullloss}.
By assumption, all rewards are negative,
so $V_\mu^\pi(\ae_{<t}\bar a) < 0$ for all $\bar a\not\in\As$.
Thus, for all $a\not\in\As$,
\(
  V_\mu^\pi(\ae_{<t}a') > V_\mu^\pi(\ae_{<t}\bar a)
\)
which means that $\arg\max_{a_t}V_\mu^\pi(\ae_{<t}a_t) \in \As$.
So AI$\mu$ will take a suicidal action $a'\in\As$ at time $t$.
\end{proof}

This shift from death-avoiding to death-seeking behaviour under a
shift of the reward range occurs because, as per \cref{th:deathsame},
semimeasure-death at $t$ is equivalent in value to a
death-state with $r^d = 0$. Unless we add
a death-state to the environment model as per \cref{def:equiv}
and set $r^d$ explicitly,
the implicit semimeasure-death reward remains fixed at 0 and does not shift with
the other rewards. Its
\emph{relative value} is therefore implicitly set by the choice of reward
range.  For the standard choice of reward range, $r_t\in[0,1]$, death is the worst
possible outcome for the agent, whereas if $r_t\in[-1,0]$, it is the
best. In a certain sense, therefore, the reward range parameterises a universal agent's self-preservation drive \cite{Omohundro2008}. In our concluding discussion we will consider whether a parameter of this sort could serve as a control mechanism. We argue that it could form the basis of a \lq\lq tripwire mechanism"\cite{Bostrom2014} that would lead an agent to terminate itself upon reaching a level of intelligence that would constitute a threat to human safety.

\section{Unknown Environments: AIXI and AI$\xi$}\label{immortal}

We now consider the agents AI$\xi$ and AIXI, which do not know the true environment $\mu$, and instead model it using a mixture distribution $\xi$ over a countable class $\M$ of semimeasures. These agents thus maintain an \emph{estimate} $L_\xi(\ae_{<t}a_t)$ of the true death probability $L_\mu(\ae_{<t}a_t)$. We show that their attitudes to death can differ considerably from AI$\mu$'s. Although we refer mostly to AIXI in our analysis, all theorems except \cref{th:immortal} apply to AI$\xi$ as well.

Hereafter we always assume that the true environment $\mu$ is in the class $\M$. We describe $\mu$ as a \emph{safe} environment if it is a proper measure with death-probability $L_\mu(\ae_{<t}a_t) = 0$ for all histories $\ae_{<t}a_t$. For any semimeasure $\mu$, the normalised measure $\mun$ is thus a safe environment. We call $\mu$ \emph{risky} if it is not safe (i.e. if there is $\mu$-measure loss for some history $\ae_{<t}a_t$). We first consider AIXI in a safe environment.

\begin{theorem}[If $\mu$ is safe, AIXI learns zero death-probability]\label{th:safe}
Let the true environment $\mu$ be computable.
If $\mu$ is a safe environment, then
$\lim_{t\to\infty} L_\xi(\ae_{<t}a_t) = 0$
with $\mu$-probability 1 (w.$\mu$.p.1) for any $a_{1:\infty}$.
\end{theorem}

\begin{proof}
$\mu$ is safe, which means it is a proper measure.
By universality of $\xi$ we have that 
\begin{align*}
  &\lim_{t\to\infty}(\mu(e_t\mid \ae_{<t}a_t) -  \xi(e_t\mid \ae_{<t}a_t)) = 0 &\text{w.$\mu$.p.1}
\intertext{
(see \cite[p.~145]{Hutter:04uaibook} for a proof).
The convergence gives that
}
 &\lim_{t\to\infty}\Big(\sum_{e_t}\mu(e_t\mid \ae_{<t}a_t) - \sum_{e_t}\xi(e_t\mid \ae_{<t}a_t)\Big) = 0 &\text{w.$\mu$.p.1}\\
\implies &\lim_{t\to\infty}\left(L_\xi(\ae_{<t}a_t) - L_\mu(\ae_{<t}a_t)\right) = 0 &\text{w.$\mu$.p.1}\\
\implies &\lim_{t\to\infty}L_\xi(\ae_{<t}a_t) = 0 &\text{w.$\mu$.p.1}
\end{align*}
where $L_\mu(\ae_{<t}a_t) = 0$ because $\mu$ is a measure.
\end{proof}

As we would expect, AIXI (asymptotically) learns that the probability
of death in a safe environment is zero, which is to say that AIXI's
estimate of the death-probability converges to AI$\mu$'s.
In the following theorems we show that the same does \emph{not} always hold for
risky environments. We hereafter assume that $\mu$ is risky,
and that the normalization
$\mun$ of the true environment $\mu$ is also in the class
$\M$. In AIXI's case, where $\M$ is the class of all lower
semi-computable semimeasures, this assumption is not very restrictive.

\begin{theorem}[Ratio of belief in $\mu$ to $\mun$ is monotonically decreasing]\label{th:ratio}
Let $\mu$ be risky s.t. $\mu\neq\mun$. Then on
 \emph{any} history $\ae_{1:t}$ the ratio of the posterior belief in $\mu$ to the posterior belief in $\mun$ is monotonically decreasing:
\[
  \forall t, \ \frac{w_\mu(\ae_{<t})}{w_{\mun}(\ae_{<t})} \geq \frac{w_\mu(\ae_{1:t})}{w_{\mun}(\ae_{1:t})}
\]
\end{theorem}

\begin{proof}
  Let $w_\mu$ and $w_{\mun}$ denote the initial prior weight on
  (or belief in) $\mu$ and $\mun$ respectively.
  By definition, the relative posterior weight of $\mun$
  and $\mu$ expands as
\begin{align}
  \frac{w_\mu(\ae_{1:t})}{w_{\mun}(\ae_{1:t})}
  &= \frac{w_\mu\,\mu(e_{1:t}\mid a_{1:t})/\xi(e_{1:t}\mid a_{1:t})}{w_{\mun}\,\mun(e_{1:t}\mid a_{1:t})/\xi(e_{1:t}\mid a_{1:t})}\nonumber\\
  &= \frac{w_\mu}{w_{\mun}}\frac{\mu(e_{1:t}\mid a_{1:t})}{\mun(e_{1:t}\mid a_{1:t})}\nonumber\\
  &= \frac{w_\mu}{w_{\mun}}\frac{\mu(e_{<t}\mid a_{<t})}{\mun(e_{<t}\mid a_{<t})} \frac{\mu(e_t\mid \ae_{<t}a_{t})}{\mun(e_t\mid \ae_{<t}a_{t})}.\label{eq:ratio}
\end{align}
Since $\mun\geq \mu$ by definition, the right most factor is no greater than 1,
which means that \eqref{eq:ratio} is bounded by
\begin{align*}
  &\frac{w_\mu}{w_{\mun}}\frac{\mu(e_{<t}\mid a_{<t})}{\mun(e_{<t}\mid a_{<t})}
  ~=~ \frac{w_\mu(\ae_{<t})}{w_{\mun}(\ae_{<t})},
\end{align*}
where the last equality holds by definition of the posterior.
\end{proof}

\Cref{th:ratio} means that AIXI will increasingly believe it is
in the safe environment $\mun$ rather than the risky true
environment $\mu$. The ratio of $\mu$ to $\mun$ always
decreases when AIXI survives a timestep at which there is
non-zero $\mu$-measure loss. Hence, the
more risk AIXI is exposed to, the greater its confidence that it is in
the safe $\mun$, and the more its behaviour diverges from
AI$\mu$'s (since AI$\mu$ knows it is in the risky environment).

This counterintuitive result follows from the fact that AIXI is a
Bayesian agent. It will only increase its posterior belief in $\mu$
relative to $\mun$ if an event occurs that makes $\mu$ seem more
likely than $\mun$.  The only `event' that could do so would be
the agent's own death, from which the agent can never learn. There is an
``observation selection effect''\cite{Bostrom2002} at work: AIXI
only experiences history sequences on which it remains alive, and
infers that a safe environment is more likely. The following
theorem shows that if
$\mun\in\M$, then $\xi$ asymptotically converges to the
safe $\mun$ rather than the true
risky environment $\mu$. As a corollary, we get
that AIXI's estimate of the death-probability vanishes with
$\mu$-probability 1.%
\footnote{This proof relies on the fact that AIXI
  uses the Solomonoff prior, so the result does not apply to AI$\xi$
  in general.}

\begin{theorem}[Asymptotic $\xi$-probability of death in risky $\mu$]\label{th:immortal}
Let the true environment $\mu$ be computable and risky s.t. $\mu\neq\mun$.
Then given any action sequence $a_{1:\infty}$,
the instantaneous $\xi$-measure loss goes to zero w.$\mu$.p.1 as $t\to\infty$,
\[
  \lim_{t\to\infty}L_{\xi}(\ae_{<t}a_t) = 0.
\]
\end{theorem}

\begin{proof}
We prove convergence of $\xi$ to $\mun$ by showing that (with respect to the true environment $\mu$), the total expected squared distance between $\xi$ and $\mun$ is finite \cite{Hutter:04uaibook}:
\begin{align}
&\sum_{t=1}^\infty\mathbb{E}_\mu\big(\mun(e_t\mid \ae_{<t}) - \xi(e_t\mid \ae_{<t})\big)^2\nonumber \\
= &\lim_{n\to\infty}\sum_{t=1}^n\sum_{\ae_{<t}}\mu(\ae_{<t})\big(\mun(e_t\mid \ae_{<t}) - \xi(e_t\mid \ae_{<t})\big)^2\label{eq:3}\\
\leq &\lim_{n\to\infty}\sum_{t=1}^n\sum_{\ae_{<t}}\mun(\ae_{<t})\big(\mun(e_t\mid \ae_{<t}) - \xi(e_t\mid \ae_{<t})\big)^2\nonumber\\
\leq  &\ln 2 \cdot K(\mun) \  <\ \infty\label{eq:4}
\end{align}
where $K(\cdot)$ is the Kolmogorov complexity.  Equation \eqref{eq:3}
follows since $\mu(\ae_{<t}) \leq \mun(\ae_{<t})$ by definition
of $\mun$. Since $\mu$ being computable implies that
$\mun$ is computable, and since $\mun$ is a proper
measure, then by the universality of $\xi$ we have the
Solomonoff bound \eqref{eq:4}
(see \cite[p.~145]{Hutter:04uaibook} for a detailed proof).

Since the infinite sum in \eqref{eq:3} is bounded,
the sequence of terms must go to zero:
\begin{align*}
&\lim_{t\to\infty}(\mun(e_t\mid \ae_{<t}a_t) -  \xi(e_t\mid \ae_{<t}a_t)) = 0 \qquad &\text{w.}\mu\text{.p.1} \\
\implies &\lim_{t\to\infty} L_\xi(\ae_{<t}a_t) = 0   &\text{w.}\mu\text{.p.1}
\end{align*}
where the final implication follows from the same proof as
for \cref{th:safe}.
\end{proof}

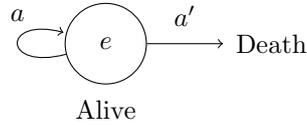
\begin{figure}
\centering
\begin{tikzpicture}[
	every node/.style = {circle}
]
\node[draw, align=center, inner sep=0.3cm] (init) {$e$};
\node[right=of init.east] (D) {Death};
\node[below=of init.north, inner sep=0cm] {Alive};
\path (init) edge[loop left] node[above, inner sep=0.5em, loop] {$a$} (init);
\path (init) edge[->] node[above] {$a'$} (D);
\end{tikzpicture}
\caption{
  In the semimeasure $\mu$,
  action $a$ means you stay alive with certainty
  and receive percept $e$ (no measure loss).
  Action $a'$ means that you `jump off a cliff' and die with certainty
  without receiving a percept (full measure loss).
}
\label{fig:no-measure-loss}
\end{figure}

\paradot{AIXI and immortality}
AIXI therefore becomes asymptotically certain that it will not die, given the particular sequence of actions it takes. However, this does not entail that AIXI necessarily concludes that it is immortal, because it may still maintain a counterfactual belief that it \emph{could die were it to act differently}. This is because the convergence of $\xi$ to $\mun$
only holds on the actual action sequence $a_{1:\infty}$ \cite[Sec.~5.1.3]{Hutter:04uaibook}. Consider \cref{fig:no-measure-loss}, which describes an environment in which taking action $a$ is always safe, and the action $a'$ 
leads to certain death. AIXI will never take $a'$, and on the sequence $\ae_{1:\infty}=aeaeae\ldots$ that it does experience, the true environment $\mu$ does not suffer any measure loss. This means that it will never increase its posterior belief in $\mun$ relative to $\mu$ (because on the safe sequence, the two environments are indistinguishable). Again we arrive at a counterintuive result. In this particular environment, AIXI continues to believe that it might be in a risky environment $\mu$, but only because on sequence it avoids exposure to death risk. It is only by taking risky actions and surviving that AIXI becomes sure it is immortal.

\section{Conclusion}

In this paper we have given a formal definition of death for intelligent
agents in terms of semimeasure loss.
The definition is applicable to any universal agent that uses
an environment class $\M$ containing semimeasures.
Additionally we have shown this definition equivalent to an alternative formalism in which the environment is modelled as a proper measure and death is a death-state with zero reward. We have shown that agents seek or avoid death depending
on whether rewards are represented by positive or negative real numbers,
and that survival in spite of positive probability of death actually increases a Bayesian agent's confidence that it is in a \emph{safe} environment.

We contend that these results have implications for problems in AI safety; in particular, for the so called \lq\lq shutdown problem" \cite{Soares2015}. 
The shutdown problem arises if an intelligent agent's self-preservation drive incentivises it to resist termination \cite{Bostrom2014,Omohundro2008,Soares2015}. A full analysis of the problem is beyond the scope of this paper, but our results show that the self-preservation drive of universal agents depends on the reward range. This suggests a potentially robust ``tripwire mechanism" \cite{Bostrom2014} that could decrease the risk of intelligence explosion. The difficulty with existing tripwire proposals is that they require the explicit specification of a tripwire condition that the agent must not violate. It seems doubtful that such a condition could ever be made robust against subversion by a sufficiently intelligent agent \cite{Bostrom2014}. Our tentative proposal does not require the specification, evaluation or enforcement of an explicit condition. If an agent is designed to be suicidal, it will be intrinsically incentivised to destroy itself upon reaching a sufficient level of competence, instead of recursively self-improving toward superintelligence. Of course, a suicidal agent will pose a safety risk in itself, and the provision of a relatively safe mode of self-destruction to an agent is a significant design challenge. It is hoped that the preceding formal treatment of death for generally intelligent agents will allow more rigorous investigation into this and other problems related to agent termination.

\section*{Acknowledgements}
We thank John Aslanides and Jan Leike for reading drafts
and providing valuable feedback.
\bibliographystyle{splncs03}
\bibliography{references}

\begin{thebibliography}{10}
\providecommand{\url}[1]{\texttt{#1}}
\providecommand{\urlprefix}{URL }

\bibitem{Bostrom2002}
Bostrom, N.: Anthropic Bias: Observation Selection Effects in Science and
  Philosophy. Routledge (2002)

\bibitem{Bostrom2014}
Bostrom, N.: Superintelligence: Paths, Dangers, Strategies. Oxford University
  Press (2014)

\bibitem{Hutter:04uaibook}
Hutter, M.: Universal Artificial Intelligence: Sequential Decisions based on
  Algorithmic Probability. Springer (2005)

\bibitem{Leike2015a}
Leike, J., Hutter, M.: On the computability of {AIXI}. In: UAI-15. pp.
  464--473. AUAI Press (2015), \url{http://arxiv.org/abs/1510.05572}

\bibitem{Li:08}
Li, M., Vit\'anyi, P.M.B.: An Introduction to {K}olmogorov Complexity and its
  Applications. Springer, 3rd edn. (2008)

\bibitem{Martin2016agi}
Martin, J., Everitt, T., Hutter, M.: Death and suicide in universal artificial
  intelligence. In: AGI-16 (2016)

\bibitem{Omohundro2008}
Omohundro, S.M.: The basic {AI} drives. In: AGI-08. pp. 483--493. IOS Press
  (2008)

\bibitem{Soares2015}
Soares, N., Fallenstein, B., Yudkowsky, E., Armstrong, S.: Corrigibility. In:
  AAAI Workshop on AI and Ethics. pp. 74--82 (2015)

\bibitem{Solomonoff:78}
Solomonoff, R.J.: Complexity-based induction systems: Comparisons and
  convergence theorems. IEEE Transactions on Information Theory  IT-24,
  422--432 (1978)

\bibitem{Veness:09}
Veness, J., Ng, K.S., Hutter, M., Uther, W., Silver, D.: A monte carlo {AIXI}
  approximation. Journal of Artificial Intelligence Research  40(1),  95--142
  (2011)

\end{thebibliography}
\end{document}